\documentclass[10pt]{article}

\usepackage{fullpage, amsmath, amsfonts, amssymb, amsthm, amstext}
\usepackage{enumitem}
\usepackage[numbib]{tocbibind}
\usepackage[calcwidth]{titlesec}
\usepackage{graphicx}
\usepackage{subfigure}
\usepackage{stmaryrd}
\usepackage{color}
\usepackage[colorlinks=true,linkcolor=blue,citecolor=blue,pagebackref]{hyperref}
\usepackage{appendix}
\usepackage{mdframed}
\usepackage{multirow}

\titleformat{\section}{\mdseries\Large}{\S\thesection}{1em}{}[\vspace{-1.6em}\rule{\titlewidth}{.25pt}]
\titleformat{\subsection}{\mdseries\scshape\large}{\S\thesubsection}{1em}{}
\titleformat{\subsubsection}{\bfseries}{\S\thesubsubsection}{1em}{}

\def \lb {{\langle}}
\def \rb {{\rangle}}

\def \R {{\mathbb R}}

\def \C {{\mathbb C}}
\def \E {{\mathbb E}}
\def \Ld {{\bf L^2}}
\def \V {{\bf V}}
\def \Diff {{\rm {Diff}}}

\def \sign {{\rm sign}}

\def \hh {{h}}
\def\cistar{\kern.2em\mbox{$\odot\kern-.67em\star\kern .4em$}}

\makeatletter
\newtheorem*{rep@theorem}{\rep@title}
\newcommand{\newreptheorem}[2]{%
\newenvironment{rep#1}[1]{%
 \def\rep@title{#2 \ref{##1}}%
 \begin{rep@theorem}}%
 {\end{rep@theorem}}}
\makeatother

\theoremstyle{theorem}

\newreptheorem{theorem}{Theorem}

\newreptheorem{lemma}{Lemma}
\newtheorem{proposition}{Proposition}

\theoremstyle{definition}

\theoremstyle{remark}

\title{Understanding Deep Convolutional Networks}

\author{St\'{e}phane Mallat\\\'{E}cole Normale Sup\'{e}rieure, CNRS, PSL\\ 
45 rue d'Ulm, 75005
  Paris, France}

\date{To appear in Philosophical Transactions A in 2016}

\begin{document}

\maketitle

\begin{abstract}
Deep convolutional networks provide state of the art
classifications and regressions results
over many high-dimensional problems.
We review their architecture, which scatters data with 
a cascade of linear filter weights and non-linearities. 
A mathematical framework is introduced to analyze 
their properties.
Computations of invariants involve 
multiscale contractions, the linearization of hierarchical symmetries,
and sparse separations.
Applications are discussed. \\
\end{abstract}

\section{Introduction}

Supervised learning is a high-dimensional interpolation problem.
We approximate a function
$f(x)$ from $q$ training samples $\{x^i, f(x^i)\}_{i \leq q}$, 
where $x$ is a data vector of very high dimension $d$.
This dimension is often
larger than $10^6$, for images or other large size signals. 
Deep convolutional neural networks 
have recently obtained remarkable experimental results \cite{nature}.  
They give state of the art performances for image
classification with thousands of complex classes \cite{Krizhevsky},
speech recognition \cite{speech}, bio-medical applications \cite{Leung},
natural language understanding \cite{language}, and in many other domains. 
They are also
studied as neuro-physiological models of vision \cite{poggio}.

Multilayer neural networks are computational learning architectures
which propagate the input data across a sequence of linear operators
and simple non-linearities.
The properties of shallow networks, with one hidden layer, are well understood
as decompositions in
families of ridge functions \cite{ridglet}. However, these approaches do not 
extend to networks with more layers.
Deep convolutional neural networks, introduced by Le Cun \cite{LeCun}, 
are implemented with linear convolutions followed by non-linearities,
over typically more than $5$ layers.
These complex programmable machines, 
defined by potentially billions of filter weights, 
bring us to a different mathematical world. 

Many researchers have pointed out that deep convolution
networks are computing progressively more powerful invariants as 
depth increases \cite{poggio,nature}, but relations with networks weights
and non-linearities are complex. This paper aims at clarifying important
principles which govern the properties 
of such networks, but their
architecture and weights may differ with applications. 
We show that computations of invariants involve 
multiscale contractions, the linearization of hierarchical symmetries,
and sparse separations. This conceptual basis 
is only a first step towards a full mathematical understanding of 
convolutional network properties.

In high dimension, $x$ has a considerable number of parameters, which
is a dimensionality curse.
Sampling uniformly a volume of dimension $d$ requires a number of samples
which grows exponentially with $d$.
In most applications, the number $q$ of training samples rather grows linearly
with $d$. It is possible to approximate $f(x)$ with so few samples,
only if $f$ has some strong regularity properties allowing to ultimately
reduce the dimension of the estimation. 
Any learning algorithm, including deep convolutional networks,
thus relies on an underlying assumption of regularity. Specifying the 
nature of this regularity is one of the core mathematical problem.

One can try to circumvent the curse of dimensionality
by reducing the variability or the dimension of $x$, without sacrificing the ability to approximate 
$f(x)$. This is done by defining a new variable $\Phi(x)$ where
$\Phi$ is a {\it contractive} operator which reduces 
the range of variations of $x$, while still {\it separating} different values
of $f$: $\Phi(x) \neq \Phi(x')$ if $f(x) \neq f(x')$. This
separation-contraction trade-off needs to be adjusted to the properties of $f$.

Linearization is a strategy used in machine learning to 
reduce the dimension with a linear
projector. A low-dimensional linear projection of $x$ can separate the values of
$f$ if this function remains constant in the direction of 
a high-dimensional linear space. 
This is rarely the case, but one can try to find $\Phi(x)$ which
linearizes high-dimensional domains where $f(x)$ remains constant.
The dimension is then reduced by applying a low-dimensional
linear projector on $\Phi(x)$. Finding such a $\Phi$ is the
dream of kernel learning algorithms, explained in Section \ref{kernel}.

Deep neural networks are more conservative. 
They progressively 
contract the space and linearize transformations along which
$f$ remains nearly constant, to preserve separation.
Such directions are defined by linear operators which belong to
groups of local symmetries, introduced in Section \ref{symmdiffsec}.
To understand the difficulty to linearize the action of high-dimensional 
groups of operators, we begin with the groups of translations and 
diffeomorphisms, which deform signals. 
They capture essential mathematical properties that are extended
to general deep network symmetries, in Section \ref{grnsdfssec}.

To linearize diffeomorphisms and 
preserve separability, Section \ref{wavelet-sec} shows that we must
separate the variations of $x$ at different scales, 
with a wavelet transform. This is implemented with
multiscale filter convolutions, which are building blocks
of deep convolution filtering. General deep network architectures
are introduced in Section \ref{DeepNet}. They iterate
on linear operators which filter and
linearly combine different channels in each network layer, 
followed by contractive non-linearities.

To understand how non-linear contractions interact with linear operators,
Section \ref{scatnsfoisdnf} begins with simpler networks which
do not recombine channels in each layer.
It defines a non-linear scattering transform, introduced in \cite{mallat-math},
where wavelets have a separation and linearization role.
The resulting 
contraction, linearization and separability properties are
reviewed. We shall see that sparsity is important for separation.

Section \ref{grnsdfssec} extends these ideas to a more general class of
deep convolutional networks.
Channel combinations provide the flexibility needed to extend
translations to larger groups of local symmetries adapted
to $f$. The network is structured by
factorizing groups of symmetries, in which case all linear operators
are generalized convolutions.
Computations are ultimately performed with filter weights, which are
learned. Their relation with groups of symmetries is
explained. A major issue is to preserve a separation margin
across classification frontiers. Deep convolutional networks have the ability
to do so, by separating network fibers 
which are progressively more invariant and specialized. 
This can give rise to invariant grandmother type neurons observed in deep networks
\cite{grandmother}.
The paper studies architectures as opposed to 
computational learning of network weights,
which is an outstanding optimization issue \cite{nature}.

{\bf Notations} $\|z\|$ is a Euclidean norm if $z$ is a vector in a Euclidean
space. If $z$ is a function in $\Ld$ then $\|z\|^2 = \int |z(u)|^2 du$.
If $z = \{z_k \}_k$
is a sequence of vectors or functions then $\|z\|^2 = \sum_k \|z_k \|^2$. 

\section{Linearization, Projection  and Separability}
\label{kernel}

Supervised learning computes an
approximation $\tilde f(x)$ of a function $f(x)$ 
from  $q$ training samples $\{x^i ,  f(x^i) \}_{i\leq q}$,  
for $x = (x(1),...,x(d))\in \Omega$.
The domain $\Omega$ is a high dimensional open 
subset of $\R^d$, not a low-dimensional manifold.
In a regression problem, $f(x)$ takes its values in $\R$, whereas
in classification its values are class indices.

\paragraph{Separation} Ideally, we would like to reduce the dimension of $x$ by computing
a low dimensional vector $\Phi(x)$ such that 
one can write $f(x) = f_0 (\Phi(x))$. It is equivalent to 
impose that if $f(x) \neq f(x')$ then
$\Phi(x) \neq \Phi(x')$. We then say that $\Phi$ {\it separates} $f$.
For regression problems, to guarantee that $f_0$ is regular, we further
impose that the separation is Lipschitz:
\begin{equation}
\label{minsdf01}
\exists \epsilon>0~~ \forall (x,x') \in \Omega^2 ~~,~~
\|\Phi(x) - \Phi(x') \| \geq \epsilon\, |f(x) - f(x')|~.
\end{equation}
It implies that $f_0$ is Lipschitz continuous: $|f_0(z)-f_0(z')| \leq 
\epsilon^{-1} |z - z'|$, for $(z,z') \in \Phi(\Omega)^2$. 
In a classification problem, $f(x) \neq f(x')$ means that $x$ and $x '$ are
not in the same class. The Lipschitz separation condition (\ref{minsdf01}) 
becomes a margin condition specifying a minimum distance across classes:
\begin{equation}
\label{minsdf010}
\exists \epsilon>0~~ \forall (x,x') \in \Omega^2 ~~,~~
\|\Phi(x) - \Phi(x') \| \geq \epsilon\, ~~\mbox{if}~~f(x) \neq f(x')~.
\end{equation}
We can try to find a linear
projection of $x$ in some space $\V$ of lower dimension $k$,
which separates $f$. It requires that
$f(x) = f(x + z)$ for all $z \in \V^\perp$, where $\V^\perp$ is the orthogonal
complement of $\V$ in $\R^d$, of dimension $d-k$.
In most cases, the final dimension $k$ can not be much smaller than $d$.

\paragraph{Linearization} An alternative strategy is to
linearize the variations of $f$ with a first change of variable
$\Phi(x) = \{ \phi_k (x) \}_{k \leq d'}$ of dimension $d'$ potentially
much larger than the dimension $d$ of $x$. We can then optimize
a low-dimensional linear projection along directions where $f$ is constant.
We say that $\Phi$ {\it separates $f$ linearly} if $f(x)$ is well approximated
by a one-dimensional projection:
\begin{equation}
\label{minsdf0}
\tilde f(x) = \lb \Phi(x) \, , w \rb = \sum_{k=1}^{d'} w_k \, \phi_k (x) ~.
\end{equation}
The regression vector $w$ is optimized by
minimizing a loss on the training data, which needs to be
regularized if $d' > q$, for example
by an $\bf l^p$ norm of $w$ with a regularization constant $\lambda$:
\begin{equation}
\label{minsdf}
\sum_{i=1}^q {\rm loss}(f(x^i) - \tilde f(x^i)) + \lambda\, \sum_{k = 1}^{d'}
|w_k|^p~.
\end{equation}
Sparse regressions are obtained with $p \leq 1$, whereas $p=2$ defines
kernel regressions \cite{Friedman}. 

Classification problems are addressed similarly, by approximating the
frontiers between classes. For example, a classification with $Q$ classes
can be reduced to $Q-1$ ``one versus all'' binary classifications.
Each binary classification is specified by an
$f(x)$ equal to $1$ or $-1$ in each class. 
We approximate $f(x)$ by 
$\tilde f(x) = \sign(\lb \Phi(x) , w \rb)$, where $w$ minimizes the
training error (\ref{minsdf}).

\section{Invariants, Symmetries and Diffeomorphisms}
\label{symmdiffsec}

We now study strategies to compute a change of variables $\Phi$ which linearizes
$f$. Deep convolutional networks operate layer per layer
and linearize $f$ progressively, as depth increases.
Classification and regression problems are addressed similarly by considering
the level sets of $f$, defined by
$\Omega_t = \{x ~:~f(x) = t \}$ if $f$ is continuous.
For classification, each level set is a particular class.
Linear separability means that one can find $w$ such
that $f(x) \approx \lb \Phi(x) , w \rb$.
If $x \in \Omega_t$ then $\lb \Phi(x) , w \rb \approx t$, so all $\Omega_t$
are mapped by $\Phi$ in different hyperplanes
orthogonal to some $w$. The change of variable linearizes the level sets of $f$.

\paragraph{Symmetries} To linearize level sets, we need to find directions along which $f(x)$
does not vary locally, and then linearize these directions in order to map
them in a linear space. It is tempting to try to do this with some
local data analysis along $x$.
This is not possible because the training set includes few 
close neighbors in high dimension. 
We thus consider simultaneously
all points $x \in \Omega$ and look for 
common directions along which $f(x)$ does not vary.
This is where groups of symmetries come in. 
Translations and diffeomorphisms will illustrate the
difficulty to linearize high dimensional symmetries, 
and provide a first mathematical ground
to analyze convolution networks architectures. 

We look for invertible operators 
which preserve the value of $f$.
The action of an operator $g$ on $x$ is written $g.x$. A global
symmetry is an invertible and often non-linear 
operator $g$ from $\Omega$ to $\Omega$,
such that $f(g.x) = f(x)$ for all $x \in \Omega$.
If $g_1$ and $g_2$ are global symmetries then $g_1.g_2$ is also
a global symmetry, so products define groups of symmetries. 
Global symmetries are usually 
hard to find. We shall first concentrate on local symmetries.
We suppose that there is a
metric $|g|_G$ which measures the distance between $g \in G$ and the identity.
A function $f$ is locally invariant to the action of $G$ if
\begin{equation}
\label{localsym}
\forall x \in \Omega~~,~~\exists C_x > 0~~,~~\forall g \in G~~\mbox{with}~~
|g|_G < C_x~~,~~f(g.x) = f(x)~.
\end{equation}
We then say that $G$ is a 
group of local symmetries of $f$. The constant $C_x$ is the local
range of symmetries which preserve $f$. 
Since $\Omega$ is a continuous subset of $\R^d$, we consider groups of operators
which transport vectors in $\Omega$ with a continuous parameter. 
They are called Lie groups if the group 
has a differential structure.

\paragraph{Translations and diffeomorphisms}
Let us interpolate the $d$ samples of $x$ and define
$x(u)$ for all  $u \in \R^n$, with $n=1,2,3$ respectively for time-series, 
images and volumetric data. The translation group $G = \R^n$ is an example of Lie group.
The action of  $g \in G = \R^n$ over $x \in \Omega$ is
$g.x(u) = x(u-g)$. The distance $|g|_G$ between $g$ and the identity
is the Euclidean norm of $g \in \R^n$. The function
$f$ is locally invariant to translations if sufficiently small
translations of $x$ do not change $f(x)$.
Deep convolutional networks compute
convolutions, because they assume that translations are local symmetries of $f$.
The dimension of a group $G$ is the number of
generators which define all group elements by products. 
For $G = \R^n$ it is equal to $n$.

Translations are not powerful symmetries because
they are defined by only $n$ variables, and $n=2$ for images.
Many image classification problems are also locally invariant to small deformations,
which provide much stronger constraints.
It means that $f$ is locally
invariant to diffeomorphisms $G = \Diff(\R^n)$, which transform $x(u)$ with a
differential warping of $u \in \R^n$. We do not know in advance what is the
local range of diffeomorphism symmetries.
For example, to classify images $x$ of hand-written digits, 
certain deformations of $x$ will
preserve a digit class but modify the class of another digit.
We shall linearize small diffeomorphims $g$. 
In a space where local symmetries are linearized, we can
find global symmetries by
optimizing linear projectors
which preserve the values of $f(x)$, and thus reduce dimensionality.

Local symmetries are linearized by finding a change of variable $\Phi(x)$
which locally linearizes the action of $g \in G$.
We say that $\Phi$ is Lipschitz continuous if
\begin{equation}
\label{addLipsch3}
\exists C > 0~,~\forall (x,g) \in \Omega \times G~~,~~
\|\Phi(g.x) - \Phi(x) \| \leq C\, |g|_G\, \|x\|~.
\end{equation}
The norm $\|x\|$ is just a normalization factor often set to $1$.
The Radon-Nikodim property proves that the map that transforms $g$ into
$\Phi(g.x)$  is almost everywhere differentiable in the sense of Gateaux.
If $|g|_G$ is small then  $\Phi(x) - \Phi(g.x)$
is closely approximated by a bounded linear operator of $g$,
which is the G\^ateaux derivative. Locally, it thus nearly remains in 
a linear space.

Lipschitz continuity over diffeomorphisms is defined relatively to a metric,
which is now defined. A small diffeomorphism acting on $x(u)$ 
can be written as a translation of $u$ by a $g(u)$:
\begin{equation}
\label{diffaction}
g.x (u) = x(u - g(u))~~ \mbox{with}~~g \in {\bf C^1} (\R^n)~.
\end{equation}
This diffeomorphism translates points by at most $\|g\|_\infty = \sup_{u \in \R^n} |g(u)|$. 
Let $|\nabla g(u)|$ be the matrix  norm of the Jacobian matrix of $g$ at $u$.
Small diffeomorphisms correspond to 
$\|\nabla g \|_\infty = \sup_u |\nabla g(u)| < 1$. 
Applying a diffeomorphism $g$ transforms two points $(u_1,u_2)$ into
$(u_1-g(u_1),u_2-g(u_2))$. Their distance is thus multiplied by a
scale factor, which is bounded above and below by
$1 \pm \|\nabla g \|_\infty$.
The distance of this diffeomorphism to the identity is defined by:
\begin{equation}
\label{diffmetrs}
|g|_{\Diff} = 2^{-J}\, \|g\|_\infty + \|\nabla g \|_\infty~.
\end{equation}
The factor $2^J$ is a local translation invariance scale. It gives
the range of translations over which small diffeomorphisms are
linearized. For $J = \infty$ the metric is globally invariant
to translations.

\section{Contractions and Scale Separation with Wavelets}
\label{wavelet-sec}

Deep convolutional networks can linearize the action of very complex non-linear
transformations in high dimensions, such as inserting glasses in images of 
faces \cite{faces}. A transformation of $x \in \Omega$ 
is a transport of $x$ in $\Omega$. 
To understand how to linearize any such transport, we shall begin
with translations and diffeomorphisms.
Deep network architectures are covariant to translations, because all
linear operators are implemented with convolutions. 
To compute invariants to translations and linearize diffeomorphisms,
we need to separate scales and apply a non-linearity. 
This is implemented with a cascade of filters computing
a wavelet transform, and a pointwise contractive non-linearity. 
Section \ref{grnsdfssec} extends these tools to general group actions.

\paragraph{Averaging} A linear operator
can compute local invariants to the action of the translation group $G$,
by averaging $x$ along the orbit $\{g.x\}_{g \in G}$,
which are translations of $x$. This is done with a convolution by an
averaging kernel $\phi_J (u) = 2^{-nJ} \phi(2^{-J} u)$ 
of size $2^J$, with $\int \phi(u)\,du = 1$:
\begin{equation}
\label{andv8sdfsdfs9}
\Phi_J x(u) = x \star \phi_J (u) ~.
\end{equation}
One can verify \cite{mallat-math} that this averaging is 
Lipschitz continuous to diffeomorphisms for all $x \in \Ld(\R^n)$,
over a translation range $2^J$.
However, it eliminates the variations of $x$ above the frequency $2^{-J}$.
If  $J = \infty$ then $\Phi_\infty x = \int x(u)\, du$,
which eliminates nearly all information.

\begin{figure*}
\center
\includegraphics[width=12cm]{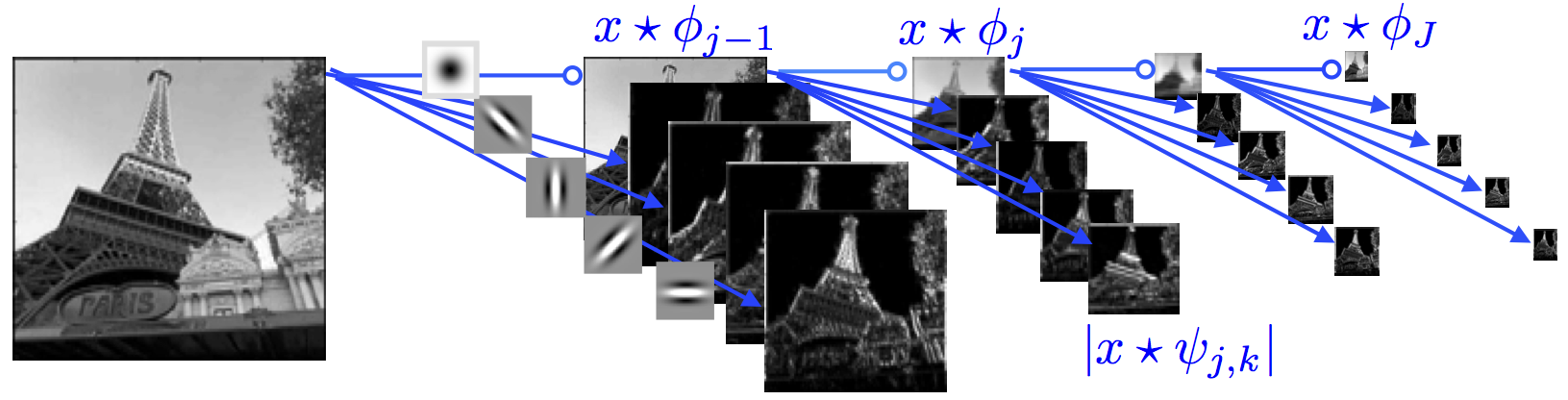}
\caption{Wavelet transform of
an image $x(u)$, computed with a cascade of convolutions 
with filters over $J = 4$ scales and $K=4$
orientations. The low-pass and $K=4$ band-pass filters are shown on the
first arrows.}
\label{figure2}
\end{figure*}

\paragraph{Wavelet transform} A diffeomorphism acts as a local translation and scaling of the variable
$u$. If we let aside translations for now, to linearize 
small diffeomorphism we must linearize this scaling action.
This is done by separating the variations of $x$ at different scales 
with wavelets. We define $K$ wavelets $\psi_k (u)$ for $u \in \R^n$.
They are regular functions with a fast decay and
a zero average $\int \psi_k (u)\, du = 0$. 
These $K$ wavelets are dilated by $2^j$:
$\psi_{j,k} (u) = 2^{-j n} \psi_k (2^{-j} u)$. 
A wavelet transform computes the local average of 
$x$ at a scale $2^J$, and variations 
at scales $2^j \geq 2^J$ with wavelet convolutions:
\begin{equation}
\label{wave1}
{\cal W} x = \{ x \star \phi_J(u) \, , \, x \star \psi_{j,k}
(u)\}_{j \leq J,  1 \leq k \leq K}~.
\end{equation}
The parameter $u$ is 
sampled on a grid such that intermediate sample values can be recovered
by linear interpolations. 
The wavelets $\psi_k$ are chosen so that ${\cal W}$ is a contractive and invertible operator,
and in order to obtain a sparse representation.
This means that $x \star \psi_{j,k}(u)$ is mostly zero besides few high
amplitude coefficients corresponding to variations of $x(u)$ 
which ``match'' $\psi_k$ at the scale $2^j$.
This sparsity plays an important role in non-linear contractions.

For audio signals, $n=1$, sparse representations are usually
obtained with at least $K = 12$ intermediate frequencies
within each octave $2^j$, which are similar to half-tone musical notes. This is done
by choosing a wavelet $\psi(u)$ having a frequency bandwidth
of less than $1/12$ octave and 
$\psi_{k} (u) = 2^{k/K} \psi (2^{-k/K} u)$ for $1 \leq k \leq K$.
For images, $n = 2$, we must discriminate image variations along
different spatial orientation. It is obtained by separating angles
$\pi k / K$, with an oriented wavelet which is rotated
$\psi_{k} (u) = \psi(r_k^{-1} u)$. Intermediate rotated wavelets are approximated
by linear interpolations of these $K$ wavelets.
Figure \ref{figure2} shows the wavelet transform of an image,
with $J = 4$ scales and $K = 4$ angles, where $x \star \psi_{j,k} (u)$ is
subsampled at intervals $2^j$. It has few 
large amplitude coefficients shown in white.

\paragraph{Filter bank} 
Wavelet transforms can be computed with a fast multiscale cascade of
filters, which is at the core of deep network architectures.
At each scale $2^j$, we define a low-pass filter $w_{j,0}$ which increases
the averaging scale from $2^{j-1}$ to $2^j$,
and band-pass filters $w_{j,k}$ which compute each wavelet:
\begin{equation}
\label{nsdf8sdf}
\phi_{j} = w_{j,0} \star  \phi_{j-1} ~~\mbox{and}~~
\psi_{j,k} = w_{j,k} \star \phi_{j-1}~.
\end{equation}
Let us write $x_{j} (u,0) = x \star \phi_j(u)$ and
 $x_{j} (u,k) = x \star \psi_{j,k}(u)$ for $k \neq 0$.
It results from (\ref{nsdf8sdf})  that for $0 < j \leq J$
and all $1 \leq k \leq K$:
\begin{equation}
\label{fionsdfsdsf}
x_{j} (u,k) = x_{j-1}(\cdot,0) \star w_{j,k} (u)~.
\end{equation}
These convolutions  may be subsampled by $2$ along $u$, in which case
$x_j (u,k)$ is sampled at intervals $2^j$ along $u$.

\paragraph{Phase removal} Wavelet coefficients $x_j(u,k) = x \star \psi_{j,k}(u)$ 
oscillate at a scale $2^j$. 
Translations of $x$ smaller than $2^j$ modifies the complex phase of
$x_j (u,k)$ if the wavelet is complex or its sign if it is real. 
Because of these oscillations, averaging $x_j$ with $\phi_J$ outputs a zero signal.
It is necessary to apply a non-linearity which removes oscillations.
A modulus $\rho(\alpha) = |\alpha|$ computes such a positive envelop.
Averaging
$\rho(x \star \psi_{j,k} (u))$ by $\phi_J$ outputs non-zero coefficients
which are locally invariant at a scale $2^J$:
\begin{equation}
\label{sdf}
\Phi_J x(u,j,k) = \rho(x \star \psi_{j,k}) \star \phi_J (u)~.
\end{equation}
Replacing the modulus by
a rectifier $\rho (\alpha) = \max(0 , \alpha)$ gives nearly the same result, up 
to a factor $2$. 
One can prove \cite{mallat-math} that this representation is 
Lipschitz continuous to actions of diffeomorphisms 
over $x \in \Ld(\R^n)$, and thus satisfies
(\ref{addLipsch3}) for the metric (\ref{diffmetrs}).
Indeed, the
wavelet coefficients of $x$  deformed by $g$ can be written as the wavelet
coefficients of $x$ with deformed wavelets. Small deformations 
produce small modifications of wavelets in $\Ld(\R^n)$, because they are localized
and regular. The resulting modifications of wavelet coefficients is of the
order of the diffeomorphism metric $|g|_{\Diff}$.

\paragraph{Contractions} A modulus and a rectifier are contractive non-linear pointwise operators:
\begin{equation}
\label{contractions}
|\rho(\alpha) - \rho(\alpha')| \leq |\alpha - \alpha'|. 
\end{equation}
However, if $\alpha = 0$ or $\alpha' = 0$ then this inequality is an equality.
Replacing $\alpha$ and $\alpha'$ by 
$x \star \psi_{j,k} (u)$ and $x' \star \psi_{j,k} (u)$ shows that distances
are much less reduced if 
$x \star \psi_{j,k} (u)$ is sparse. Such contractions 
do not reduce as
much the distance between sparse signals and other signals. 
This is illustrated by reconstruction examples in Section \ref{scatnsfoisdnf}.

\paragraph{Scale separation limitations} 
The local multiscale invariants in (\ref{sdf}) have dominated
pattern classification applications for music, speech and images, until
$2010$. It is called
{\it Mel-spectrum} for audio  \cite{shamma} and SIFT type 
feature vectors \cite{SIFT} in images.
Their limitations comes from the loss of information produced by 
the averaging by $\phi_J$ in (\ref{sdf}). To reduce this loss,
they are computed at short time scales $2^J \leq 50 ms$ in audio signals, 
or over small image patches $2^{2J} = 16^2$ pixels.
As a consequence, they do not capture large scale structures, which are
important for classification and regression problems. To build a rich set
of local invariants at a large scale $2^J$,
it is not sufficient to separate scales with wavelets, we must also 
capture scale interactions.

A similar issue appears
in physics to characterize the interactions of complex systems.
Multiscale separations are used to reduce the parametrization
of classical many body systems, for example with multipole methods 
\cite{Greengard}. However, it does not apply to
complex interactions, as in quantum systems. 
Interactions across scales, between small and larger structures, must
be taken into account. Capturing these interactions with low-dimensional models
is a major challenge. We shall see that deep neural networks and
scattering transforms provide high order coefficients which partly characterize
multiscale interactions.

\section{Deep Convolutional Neural Network Architectures}
\label{DeepNet}

\begin{figure*}
\center
\includegraphics[natwidth=2280,natheight=1450,width=12cm]{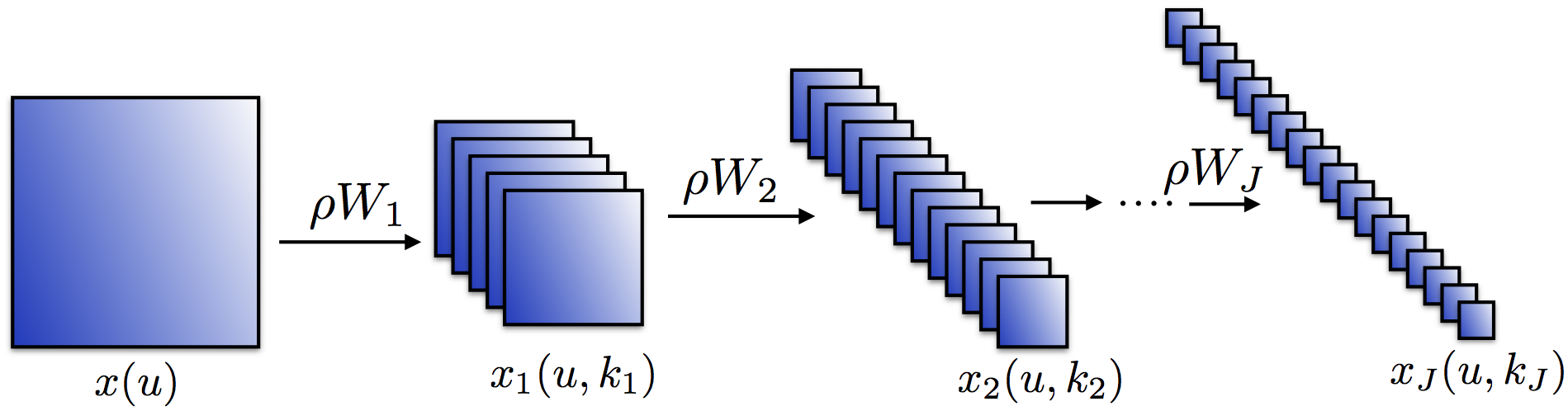}
\caption{A convolution network iteratively computes each layer $x_{j}$ by 
transforming the previous layer $x_{j-1}$, with a linear operator $W_j$ and a pointwise non-linearity $\rho$. }
\label{figure1}
\end{figure*}

Deep convolutional networks are computational architectures introduced
by Le Cun \cite{LeCun}, providing remarkable regression and 
classification results in high dimension \cite{nature,Krizhevsky,speech}.
We describe these architectures illustrated by Figure \ref{figure1}.
They iterate over linear operators $W_j$ including convolutions,
and predefined pointwise non-linearities.

A convolutional network takes in input a signal $x(u)$, which is here an image. 
An internal network layer $x_j (u,k_j)$ at a depth $j$
is indexed by the same translation variable $u$,
usually subsampled, and a channel index $k_j$. 
A layer $x_{j}$ is computed from $x_{j-1}$ by applying a linear operator $W_j$
followed by a pointwise non-linearity $\rho$:
\[
x_{j}  = \rho W_j x_{j-1} ~.
\]
The non-linearity $\rho$ 
transforms each coefficient $\alpha$ of the array $W_j x_{j-1}$, and satisfies
the contraction condition (\ref{contractions}).
A usual choice is the rectifier
$\rho(\alpha) = \max(\alpha, 0)$ for $\alpha \in \R$, but it can
also be a sigmoid, or a 
modulus $\rho(\alpha) = |\alpha|$ where $\alpha$ may be complex. 

Since most classification and regression functions $f(x)$
are invariant or covariant to translations, 
the architecture imposes that $W_j$ is covariant to translations. The
output is translated if the input is translated. Since $W_j$ is linear,
it can thus be written as a sum of convolutions:
\begin{equation}
\label{filtering}
[W_j x_{j-1}] (u,k_j) =  
\sum_{k} \sum_{v} x_{j-1}(v,k)\, w_{j, k_j} (u-v,k) =
\sum_{k} [x_{j-1}(\cdot,k) \star w_{j, k_j} (\cdot,k)] (u) ~.
\end{equation}
The variable $u$ is usually subsampled. For a fixed $j$, all
filters $w_{j,k_{j}}(u,k)$ have the same support width along $u$,
typically smaller than $10$. 

The operators $\rho\, W_j$ propagates the input signal $x_0 = x$ 
until the last layer $x_J$. 
This cascade of spatial convolutions defines translation covariant
operators of progressively
wider supports as the depth $j$ increases. Each $x_j (u,k_j)$
is a non-linear function of $x(v)$, for $v$ in a square centered at $u$,
whose width $\Delta_j$ does not depend upon $k_j$. The width
$\Delta_j$ is the spatial scale of a layer $j$. 
It is equal to $2^j\, \Delta$ if all filters $w_{j,k_{j}}$ have a
width $\Delta$ and the convolutions (\ref{filtering}) are 
subsampled by $2$.

Neural networks include many side tricks. 
They sometimes normalize the amplitude of
$x_j (v,k)$, by dividing it by the norm of all coefficients 
$x_j (v,k)$ for $v$ in a neighborhood of $u$. This eliminates
multiplicative amplitude variabilities. Instead of subsampling 
(\ref{filtering}) on a regular grid, a max pooling 
may select the largest coefficients over each sampling cell.
Coefficients may also be modified by subtracting a constant adapted
to each coefficient. When applying a rectifier $\rho$, this constant
acts as a soft threshold, which increases sparsity.
It is usually observed that inside network coefficients
$x_j (u,k_j)$ have a sparse activation. 

The deep network output 
$x_J = \Phi_J (x)$ is provided to a classifier, 
usually composed of fully connected neural network layers
\cite{nature}. Supervised deep learning algorithms
optimize the filter values $w_{j,k_{j}}(u,k)$ in order
to minimize the average classification or regression error on the training
samples $\{x^i\, , \, f(x^i) \}_{i \leq q}$. 
There can be more than $10^8$ variables in a network \cite{nature}. 
The filter update is done with a back-propagation algorithm,
which may be computed with a stochastic gradient descent, with
regularization procedures such as dropout.
This high-dimensional optimization is non-convex, but
despite the presence of many local minima, the regularized
stochastic gradient descent converges to a local minimum providing good
accuracy on test examples \cite{benarous}. The rectifier non-linearity
$\rho$ is usually preferred
because the resulting optimization has a better convergence.
It however requires a large number of training examples. 
Several hundreds of examples per class are usually
needed to reach a good performance.

Instabilities have been observed in some network architectures
\cite{adversarial}, where additions of small perturbations on
$x$ can produce large variations of $x_J$. It happens when the norms of
the matrices $W_j$ are larger than $1$, and hence amplified
when cascaded. However, deep network also have a strong form of stability
illustrated by transfer learning \cite{nature}. A deep network
layer $x_J$ optimized on particular training databasis, can 
approximate different classification functions,
if the final classification layers are trained on a new databasis. 
This means that it has learned stable structures, which can be
transferred across similar learning problems.

\section{Scattering on the Translation Group}
\label{scatnsfoisdnf}

A deep network alternates linear operators $W_j$ and 
contractive non-linearities $\rho$. To analyze the properties
of this cascade, we begin with a simpler architecture, where $W_j$ does not
combine multiple convolutions across channels in each layer.
We show that such network coefficients are obtained through convolutions
with a reduced number of equivalent wavelet filters. It defines
a scattering transform \cite{mallat-math} whose contraction and
linearization properties are reviewed.
Variance reduction and loss of information are studied with reconstructions of
stationary processes.

\paragraph{No channel combination} 
Suppose that $x_{j}(u,k_{j})$ is computed by convolving 
a single channel $x_{j-1} (u,k_{j-1})$  along $u$:
\begin{equation}
\label{nsdfs}
x_{j} (u, k_{j}) = \rho\Big(x_{j-1}(\cdot,k_{j-1}) \star w_{j,\hh} (u)\Big)~~\mbox{with}~~k_{j} = (k_{j-1}, \hh)~.
\end{equation}
It corresponds to a deep network filtering (\ref{filtering}), where
filters do not combine several channels.
Iterating on $j$ defines a convolution tree, as opposed to a full network. 
It results from  (\ref{nsdfs}) that
\begin{equation}
\label{cascasdnfsd}
x_J (u,k_J) = \rho(\rho(\rho(\rho(x \star w_{1,\hh_1})
\star ... )\star w_{J-1,\hh_{J-1}}) \star w_{J,\hh_J}) ~.
\end{equation}
If $\rho$ is a
rectifier $\rho(\alpha) = \max(\alpha,0)$
or a modulus $\rho(\alpha) = |\alpha|$ then 
$\rho(\alpha) = \alpha$ if $\alpha \geq 0$. 
We can thus remove this non-linearity at the output of an averaging
filter $w_{j,\hh}$. Indeed
this averaging filter is applied
to positive coefficients and thus computes positive coefficients,
which are not affected by $\rho$.  
On the contrary, if $w_{j,\hh}$ is a band-pass
filter then the convolution with $x_{j-1}(\cdot,k_{j-1})$
has alternating signs or a complex
phase which varies. The non-linearity 
$\rho$ removes the sign or the phase, which has a strong contraction
effect. 

\paragraph{Equivalent wavelet filter} Let $m$ be the number of band-pass filters 
$\{ w_{j_n,\hh_{j_n}} \}_{1 \leq n \leq m}$ in the convolution
cascade (\ref{cascasdnfsd}). All other filters are thus
low-pass filters. If we remove 
$\rho$ after each low-pass filter, 
we get $m$ equivalent band-pass filters:
\begin{equation}
\label{eqndf08sdfsd}
\psi_{j_n,k_n} (u) = w_{j_{n-1}+1,\hh_{j_{n-1}+1}} \star ... 
 \star w_{j_n , \hh_{j_n}} (u)~.
\end{equation}
The cascade of $J$ convolutions (\ref{cascasdnfsd}) is reduced
to $m$ convolutions with these equivalent filters
\begin{equation}
\label{deandfs}
x_J (u,k_J) = 
\rho(\rho(...\rho(\rho(x \star \psi_{j_1,k_1}) \star \psi_{j_2,k_2})... \star 
\psi_{j_{m-1},k_{m-1}}) \star 
\psi_{J,k_J} (u))~,
\end{equation}
with $0 < j_1 < j_2 <...<j_{m-1} < J$.
If the final filter $w_{J,\hh_J}$ at the depth $J$ is a low-pass filter then
$\psi_{J,k_J} = \phi_J$ is an equivalent low-pass filter. In this case,
the last non-linearity $\rho$ can also be removed, which gives
\begin{equation}
\label{deandfs2}
x_J(u,k_J) = \rho(\rho(...\rho(\rho(x \star \psi_{j_1,k_1}) \star \psi_{j_2,k_2})... \star 
\psi_{j_{m-1},k_{m-1}}) 
\star \phi_J (u)~.
\end{equation}
The operator $\Phi_J x = x_J$ is a wavelet scattering transform,
introduced in \cite{mallat-math}.
Changing the network filters $w_{j,\hh}$ modifies the equivalent band-pass
filters $\psi_{j,k}$.
As in the fast wavelet transform algorithm (\ref{fionsdfsdsf}),
if $w_{j,\hh}$ is a rotation of a dilated filter $w_j$ then
$\psi_{j,\hh}$ is a dilation and rotation of
a single mother wavelet $\psi$.

\paragraph{Scattering order} The order $m = 1$ coefficients 
$x_J(u,k_J) = \rho( x \star \psi_{j_1,k_1}) \star \phi_J (u)$ are 
the wavelet coefficient computed in (\ref{sdf}). 
The loss of information due to averaging
is now compensated by higher order coefficient.
For $m = 2$, 
$\rho(\rho(x \star \psi_{j_1,k_1}) \star  \psi_{j_2,k_2}) \star \phi_J$ 
are complementary invariants. They measure interactions
between variations of $x$ at a scale $2^{j_1}$, within 
a distance $2^{j_2}$, and along orientation or frequency bands
defined by $k_1$ and $k_2$. 
These are scale interaction coefficients, missing from first order coefficients.
Because $\rho$ is strongly contracting,
order $m$ coefficients have an amplitude
which decrease quickly as $m$ increases \cite{mallat-math,waldspurger}.
For images
and audio signals, the energy of scattering coefficients becomes negligible
for $m \geq 3$. Let us emphasize that
the convolution network depth is $J$, whereas $m$ is the number of effective
non-linearity of an output coefficient.

\paragraph{Diffeomorphism continuity} Section \ref{wavelet-sec} explains that a wavelet transform defines 
representations which are Lipschitz continuous to actions of diffeomorphisms.
Scattering coefficients up to the order $m$ are computed by applying
$m$ wavelet transforms. One can prove \cite{mallat-math} that it thus
defines a representation which is Lipschitz continuous to the 
the action of diffeomorphisms. There exists $C > 0$ such that
\[
\forall (g,x) \in \Diff(\R^n) \times \Ld(\R^n)~~,~~
\|\Phi_J (g.x) - \Phi_J x \| \leq C\, m\,\Big(2^{-J} \|g\|_\infty + \|\nabla g \|_\infty\Big)\, \|x\|~,
\]
plus a Hessian term which is neglected. This result is proved
in \cite{mallat-math} for $\rho(\alpha) = |\alpha|$, but it remains valid for
any contractive pointwise operator such as rectifiers
$\rho(\alpha) = \max(\alpha,0)$. 
It relies
on commutation properties of wavelet transforms and diffeomorphisms. 
It shows that the action of small diffeomorphisms is linearized over
scattering coefficients. 

\paragraph{Classification} Scattering vectors are restricted to coefficients of order $m \leq 2$,
because their amplitude is negligible beyond.
A translation scattering $\Phi_J x$ is well adapted to
classification problems where the main source of intra-class variability
are due to translations, to small deformations, or to ergodic
stationary processes. 
For example, intra-class variabilities of 
hand-written digit images are essentially due to translations and deformations. 
On the MNIST digit data basis \cite{Bruna},
applying a linear classifier to scattering coefficients $\Phi_J x$ 
gives state 
of the art classification errors.
Music or speech classification over short time intervals
of $100ms$ can be modeled by ergodic stationary processes.
Good music and speech classification results are then
obtained with a scattering transform \cite{Anden}.
Image texture classification are also problems where intra class variability
can be modeled by ergodic stationary processes.
Scattering transforms give state of the art
results over a wide range of image texture databases
\cite{Bruna,sifre}, compared to other descriptors including
power spectrum moments. 
Softwares can be retrieved at
{\it www.di.ens.fr/data/software}.

\paragraph{Stationary processes} To analyze the information loss,
we now study the reconstruction of $x$ from its scattering coefficients, in a
stochastic framework where $x$ is a stationary process. 
This will raise variance and
separation issues, where sparsity plays a role. 
It also demonstrates the importance
of second order scale interaction terms, to capture non-Gaussian geometric
properties of ergodic stationary processes. 
Let us consider scattering coefficients of order $m$
\begin{equation}
\label{estnasfs}
\Phi_J x(u,k) = \rho(...\rho(\rho(x \star \psi_{j_1,k_1})\star \psi_{j_2,k_2})
... \star 
\psi_{j_{m},k_{m}}) 
\star \phi_J (u)~,
\end{equation}
with $\int \phi_J (u) du = 1$. 
If $x$ is a stationary process then 
$\rho(...\rho(x \star \psi_{j_1,k_1})... \star \psi_{j_{m},k_{m}})$
remains stationary because convolutions and pointwise operators
preserve stationarity. The spatial averaging by $\phi_J$ provides a non-biased
estimator of the expected value of $\Phi_J x(u,k)$, 
which is a scattering moment:
\begin{equation}
\label{nnsdf}
\E (\Phi_J x(u,k)) = \E \Big(
\rho(...\rho(\rho(x \star \psi_{j_1,k_1})\star \psi_{j_2,k_2})
... \star 
\psi_{j_{m},k_{m}})  \Big) ~.
\end{equation}
If $x$ is a slow mixing process, which is a weak ergodicity assumption, then 
the estimation variance $\sigma_J^2 = \|\Phi_J x - \E(\Phi_J x)\|^2$
converges to zero \cite{brunastat} when $J$ goes to $\infty$. 
Indeed, $\Phi_J$ is computed by iterating on
contractive operators, which average an ergodic stationary process $x$ 
over progressively larger scales. 
One can prove that scattering moments characterize complex
multiscale properties of fractals and multifractal processes, such 
as Brownian motions, Levi processes or Mandelbrot cascades
\cite{brunaAnals}.

\begin{figure*}
\begin{tabular}{lllll}
\includegraphics[width=2.5cm]{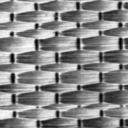}&
\includegraphics[width=2.5cm]{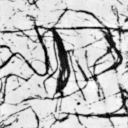}&
\includegraphics[width=2.5cm]{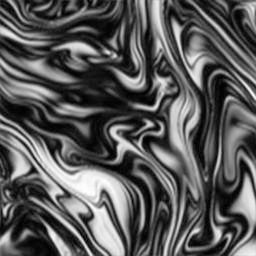}&
\includegraphics[width=2.5cm]{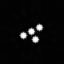}&
\includegraphics[width=2.5cm]{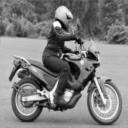}\\
\includegraphics[width=2.5cm]{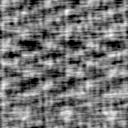}&
\includegraphics[width=2.5cm]{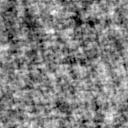}&
\includegraphics[width=2.5cm]{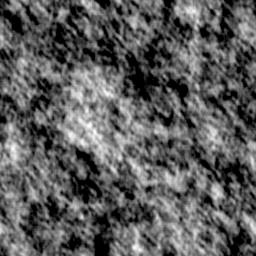}&
 &\\
\includegraphics[width=2.5cm]{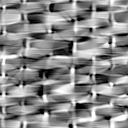}&
\includegraphics[width=2.5cm]{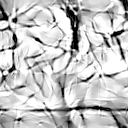}&
\includegraphics[width=2.5cm]{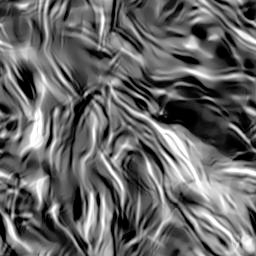}&
\includegraphics[width=2.5cm]{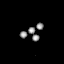}&
\includegraphics[width=2.5cm]{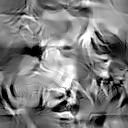}
\end{tabular}
\caption{First row: original images. Second row: realization of
a Gaussian process with same second covariance moments.
Third row: reconstructions from 
first and second order scattering coefficients.}
\label{textures}
\end{figure*}

\paragraph{Inverse scattering and sparsity} 
Scattering transforms are generally not invertible but 
given $\Phi_J (x)$ one can can compute vectors $\tilde x$ such that
$\|\Phi_J (x) - \Phi_J (\tilde x)\| \leq \sigma_J$. 
We initialize $\tilde x_0$ as a Gaussian white noise realization,
and iteratively update $\tilde x_n$ by reducing
$\|\Phi_J (x) - \Phi_J (\tilde x_n)\|$ with a gradient descent algorithm, 
until it reaches $\sigma_J$ \cite{brunastat}. Since $\Phi_J (x)$ is not
convex, there is no guaranteed convergence, but numerical reconstructions
converge up to a sufficient precision. The recovered $\tilde x$
is a stationary process having nearly the same scattering moments as $x$,
whose properties are similar to a maximum entropy process for fixed
scattering moments \cite{brunastat}.

Figure \ref{textures} shows several examples of images $x$ with
$N^2$ pixels. 
The first three images are realizations of ergodic stationary textures.
The second row gives realizations of stationary Gaussian 
processes having the same $N^2$
second order covariance moments as the top textures.
The third column shows the vorticity field of a two-dimensional turbulent fluid. 
The Gaussian realization is thus a Kolmogorov type model, which does not 
restore the filament geometry.
The third row gives reconstructions from scattering coefficients, limited to
order $m \leq 2$.
The scattering vector is computed at the maximum scale $2^J = N$, with
wavelets having $K = 8$ different orientations. It 
is thus completely invariant to translations.
The dimension of
$\Phi_{J} x$ is about $(K \log_2 N)^2/2 \ll N^2$.
Scattering moments 
restore better texture geometries than the Gaussian models 
obtained with $N^2$ covariance moments. This geometry is mostly captured
by second order scattering coefficients,
providing scale interaction terms. Indeed, 
first order scattering moments can only 
reconstruct images which are
similar to realizations of Gaussian processes.
First and second order scattering moments also provide good models
of ergodic audio textures \cite{brunastat}.

The fourth image has very sparse wavelet coefficients. In this case the image
is nearly perfectly restored by its scattering coefficients, up to a random translation. The reconstruction is centered for comparison. 
Section \ref{wavelet-sec} explains that if wavelet coefficients are
sparse then a rectifier or an absolute value contractions $\rho$ 
does not contract as much distances with other signals. Indeed, 
$|\rho(\alpha) - \rho(\alpha')| = |\alpha - \alpha'|$ if $\alpha=0$ or
$\alpha'=0$. Inverting a scattering transform is a non-linear inverse
problem, which requires to recover a lost phase information. 
Sparsity has an important role on such phase recovery problems
\cite{waldspurger}. Translating randomly the last motorcycle image defines
a non-ergodic stationary process, whose wavelet coefficients are not as sparse. 
As a result, the reconstruction from a
random initialization is very different, and does not 
preserve patterns which are important for most classification tasks.
This is not surprising since there is much less scattering coefficients than
image pixels. If we reduce $2^J$ so that the number of 
scattering coefficients reaches the number of pixels then the reconstruction is of good quality, but there is little variance reduction.

Concentrating on the translation group is not so effective to reduce variance
when the process is not translation ergodic. Applying wavelet filters 
can destroy important structures which are not sparse over wavelets. 
Next section addresses both issues. 
Impressive texture synthesis results have been obtained with 
deep convolutional networks trained on image data bases
\cite{bethdge}, but with much more output coefficients. 
Numerical reconstructions \cite{fergus} also show that one
can also recover complex patterns, such as
birds, airplanes, cars, dogs, ships,
if the network is trained to recognize the corresponding image classes.
The network keeps some form of memory of important
classification patterns.

\section{Multiscale Hierarchical Convolutional Networks}
\label{grnsdfssec}

Scattering transforms on the translation group are restricted
deep convolutional network architectures, which suffer from
variance issues and loss of information. We shall explain why 
channel combinations provide the flexibility needed to avoid some of these 
limitations. We analyze a general class of
convolutional network architectures by extending the tools previously introduced. 
Contractions and invariants to translations are
replaced by contractions along groups of local symmetries adapted
to $f$, which are defined by parallel transports in each network layer. 
The network is structured by
factorizing groups of symmetries, as depth increases. It
implies that all linear operators can be written as generalized convolutions
across multiple channels. To preserve the classification margin, wavelets must
also be replaced by adapted filter weights, which separate discriminative patterns in multiple
network fibers.

\paragraph{Separation margin}
Network layers $x_{j} = \rho W_{j} x_{j-1}$ are computed with
operators $\rho W_j$ which contract and separate components of $x_j$.
We shall see that $W_j$ also needs to prepare
$x_j$ for the next transformation $W_{j+1}$, 
so consecutive operators $W_j$ and $W_{j+1}$ are strongly dependant.
Each $W_{j}$ is a contractive linear operator, 
$\|W_{j} z\| \leq \|z\|$ to reduce the space volume, and 
avoid instabilities when cascading such operators \cite{adversarial}.
A layer $x_{j-1}$ must separate $f$ so that we can write
$f(x) = f_{j-1} (x_{j-1})$ for some function $f_{j-1}(z)$.
To simplify explanations, we 
concentrate on classification, where 
separation is an $\epsilon > 0$ 
margin condition:
\begin{equation}
\label{sdnfsdfsdfe0}
\forall (x,x') \in \Omega^2~~,~~
\| x_{j-1} -  x'_{j-1}  \| \geq \epsilon ~~
\mbox{if $f(x) \neq f(x')$}~.
\end{equation}
The next layer $x_{j} = \rho W_{j} x_{j-1}$ lives in a contracted space
but it must also satisfy
\begin{equation}
\label{sdnfsdfsdfe0987sd}
\forall (x,x') \in \Omega^2~~,~~
\|\rho W_{j} x_{j-1} - \rho W_{j} x'_{j-1}  \| \geq \epsilon \, ~~
\mbox{if $f (x) \neq f (x')$}~.
\end{equation}
The operator $W_{j}$ computes a linear projection
which preserves this margin condition,
but the resulting dimension reduction is limited.
We can further contract the space non-linearly with $\rho$. To preserve the
margin, it must reduce distances along non-linear displacements which transform 
any $x_{j-1}$ into an $x'_{j-1}$ which is in the same class.

\paragraph{Parallel transport}
Displacements which preserve classes are defined by local symmetries (\ref{localsym}),
which are transformations $\bar g$ such that 
$f_{j-1}(x_{j-1}) = f_{j-1}(\bar g.x_{j-1})$.
To define a local invariant to 
a group of transformations $G$,
we must process the orbit
$\{ \bar g. x_{j-1} \}_{g \in G}$. 
However, $W_j$ is applied to $x_{j-1}$ not on the non-linear transformations
$\bar g.x_{j-1}$ of $x_{j-1}$. The key idea is that
a deep network can proceed in two steps. 
Let us write $x_{j} (u,k_j) = x_j(v)$ with $v \in P_{j}$. 
First, $\rho W_j$ computes an approximate mapping of such an
orbit $\{ \bar g. x_{j-1} \}_{\bar g \in G}$ into
a parallel transport in $P_j$, which moves coefficients of $x_j$.  
Then $W_{j+1}$ applied to $x_j$ is filtering the orbits of this parallel transport.
A parallel transport is defined by operators
$g \in G_j$ acting on $v \in P_j$, and we write
\[
\forall (g,v) \in G_j \times P_j~~,~~g.x_j (v) = x_j (g.v)~.
\]
The operator $W_j$ is defined 
so that $G_j$ is a group of local symmetries: $f_j (g.x_j) = f_j (x_j)$ 
for small $|g|_{G_j}$. This is obtained if a transport of $x_j = W_j x_{j-1}$
by $g \in G_j$ corresponds to the action of a
local symmetry $\bar g$ of $f_{j-1}$ on $x_{j-1}$:
\begin{equation}
\label{Onsdf09sdf}
g. [\rho W_j x_{j-1}]  = \rho W_j  [\bar g.x_{j-1}]~.
\end{equation}
By definition $f_j (x_j) = f_{j-1} (x_{j-1}) = f(x)$. 
Since $f_{j-1} (\bar g . x_{j-1}) = f_{j-1} (x_{j-1})$ it results 
from (\ref{Onsdf09sdf}) that
$f_j (g.x_j) = f_j (x_j)$.

\begin{figure*}
\center
\includegraphics[width=6cm]{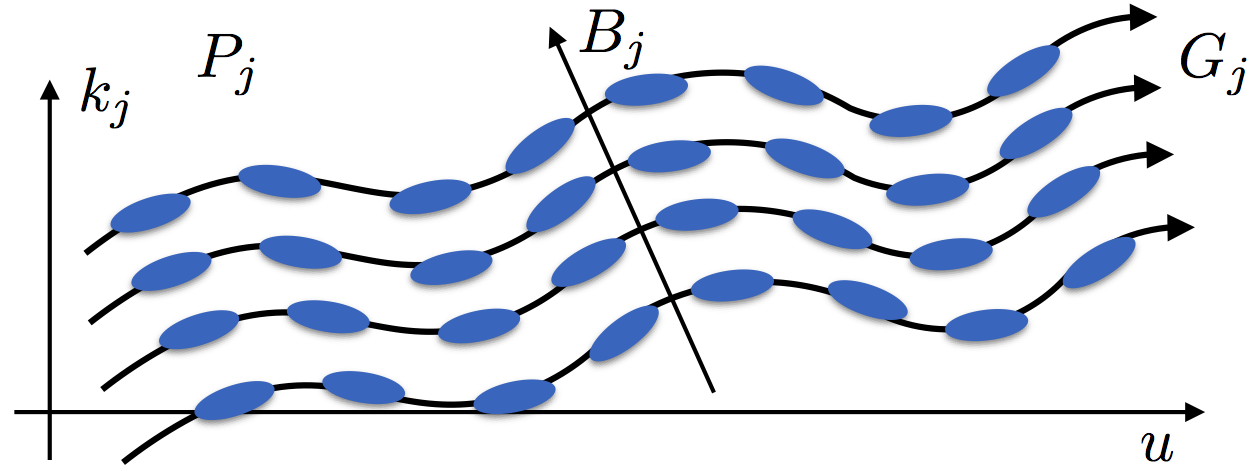}
\caption{A multiscale hierarchical networks computes convolutions along the fibers of
a parallel transport. It is defined by a group $G_j$ of symmetries acting on
the index set $P_j$ of a layer $x_j$. 
Filter weights are transported along fibers.}
\label{figure4}
\end{figure*}

The index space $P_{j}$ is called
a $G_{j}$-principal fiber bundle in differential geometry \cite{Petitot},
illustrated by Figure \ref{figure4}. 
The orbits of $G_j$ in $P_j$ are fibers,
indexed by the equivalence classes $B_{j} = P_{j} / G_{j}$.
They are
globally invariant to the action of $G_j$, and play an important role
to separate $f$. Each fiber is indexing a continuous Lie group, 
but it is sampled along $G_j$ at 
intervals such that values of $x_j$ can be interpolated in between.
As in the translation case, 
these sampling intervals depend upon the local invariance of $x_j$, which 
increases with $j$.

\paragraph{Hierarchical symmetries} In a hierarchical convolution network, we further impose that
local symmetry groups are growing with depth, and can be factorized:
\begin{equation}
\label{enbmdosdfbsd}
\forall j \geq 0~~,~~G_{j} = G_{j-1} \rtimes H_{j}.
\end{equation}
The hierarchy begins for $j = 0$ by the translation group $G_0 = \R^n$,
which acts on $x(u)$ through the spatial variable $u \in \R^n$. 
The condition (\ref{enbmdosdfbsd}) 
is not necessarily satisfied by general deep networks,
besides $j = 0$ for translations. 
It is used
by joint scattering transforms \cite{sifre,andenLostanlen} 
and has been proposed for unsupervised convolution network learning \cite{szlam}.
Proposition \ref{props1} proves that this hierarchical embedding
implies that each $W_{j}$ is a convolution on $G_{j-1}$.

\begin{proposition}
\label{props1}
The group embedding (\ref{enbmdosdfbsd}) implies that $x_j$ can be indexed by
$(g,\hh,b) \in G_{j-1} \times H_j \times B_j$ and there exists
$w_{j,\hh.b} \in \C^{P_{j-1}}$ such that
\begin{equation}
\label{Consdfoushdf}
x_{j} (g,\hh,b) = \rho \Big( \sum_{v' \in P_{j-1}} x_{j-1} (v') \, w_{j,\hh.b} (g^{-1}. v')
\Big) = 
\rho \Big( x_{j-1} \star^{j-1}  w_{j,\hh.b} (g) \Big) ~, 
\end{equation}
where $\hh.b$ transports $b \in B_j$ by $\hh \in H_j$ in $P_j$.
\end{proposition}

\begin{proof}
We write
$x_{j} = \rho W_{j} x_{j-1}$ as inner products with row vectors:
\begin{equation}
\label{definHj}
\forall v \in P_j~~,~~x_{j} (v) = \rho \Big( \sum_{v' \in P_{j-1}} x_{j-1} (v') \, w_{j,v} (v')
\Big) = \rho \Big( \lb x_{j-1} \,,\, w_{j,v} \rb \Big)~.
\end{equation}
If $\bar g \in G_j$ then 
$\bar g. x_{j} (v) = x_j (\bar g . v) = \rho ( \lb  x_{j-1} \,,\, w_{j,\bar g.v} \rb$.
One can write $w_{j,v} =  w_{j,\bar g. b}$ with $\bar g \in G_j$
and $b \in B_j = P_j / G_j$. 
If $G_{j} = G_{j-1} \rtimes H_j$ then
$\bar g \in G_j$ can be decomposed 
into $\bar g = (g, \hh) \in G_{j-1} \rtimes H_j$, where
$g.x_j = \rho (\lb g.x_{j-1},w_{j,b} \rb)$. 
But $g. x_{j-1}(v') = x_{j-1} (g. v')$ so with a change of variable 
we get
$w_{j,g.b} (v') = w_{j,b}(g^{-1}. v')$. 
Hence $w_{j, \bar g.b} (v') = w_{j,(g,l).b}(v) = \hh.w_{j,\hh.b}(g^{-1}. v')$.
Inserting this filter expression in (\ref{definHj}) proves (\ref{Consdfoushdf}).
\end{proof}

This proposition proves that $W_j$ is a convolution along the fibers
of $G_{j-1}$ in $P_{j-1}$.
Each $w_{j,\hh.b}$ is a transformation of an elementary
filter $w_{j,b}$ by a group of local symmetries $\hh \in H_j$ so that 
$f_j (x_j (g,\hh,b))$ remains
constant when $x_j$ is locally transported along $\hh$. 
We give below several examples
of groups $H_j$ and filters $w_{j,\hh.b}$. 
However, learning algorithms compute
filters directly, with no prior knowledge on the group $H_j$. 
The filters $w_{j,\hh.b}$ can be optimized so that variations of
$x_j (g,\hh,b)$ along $\hh$ captures a large variance of
$x_{j-1}$ within each class. Indeed, this variance is then reduced by
the next $\rho W_{j+1}$.
The generators of $H_j$ can
be interpreted as {\it principal symmetry generators}, by analogy 
with the principal directions of a PCA.

\paragraph{Generalized scattering} The scattering convolution along translations
(\ref{nsdfs}) is replaced in
(\ref{Consdfoushdf}) by a convolution along $G_{j-1}$, which 
combines different layer channels. Results for translations can 
essentially be extended to the general case.
If $w_{j,\hh.b}$ is an averaging filter then it computes positive
coefficients, so the non-linearity $\rho$
can be removed. 
If each filter $w_{j,\hh.b}$ has a support in
a single fiber indexed by $b$, as in Figure \ref{figure4},
then $B_{j-1} \subset B_j$.
It defines a  generalized scattering transform, which is 
a structured multiscale hierarchical convolutional network such that
$G_{j-1} \rtimes H_j = G_j$ and $B_{j-1} \subset B_j$.
If $j=0$ then $G_0 = P_0 = \R^n$ so $B_0$ is reduced to $1$ fiber.

As in the translation case, we need to linearize small deformations in
$\Diff(G_{j-1})$, which include much more local symmetries than
the low-dimensional group $G_{j-1}$.
A small diffeomorphism $g \in \Diff(G_{j-1})$ 
is a non-parallel transport along the fibers of $G_{j-1}$ in $P_{j-1}$, 
which is a perturbation of a parallel
transport. It modifies distances between pairs of points
in $P_{j-1}$ by scaling factors. 
To linearize
such diffeomorphisms, we must use localized filters whose supports have
different scales. Scale parameters are typically
different along the different generators
of $G_{j-1} = \R^n \rtimes H_1 \rtimes ...\rtimes H_{j-1}$. 
Filters can be constructed with wavelets dilated at different scales,
along the generators of each group $H_k$ for $1 \leq k \leq j$.
Linear dimension reduction mostly results from this
filtering.
Variations at fine scales may be eliminated, so that $x_j (g,\hh,b)$
can be coarsely sampled along $g$.

\paragraph{Rigid movements} For small $j$, the local symmetry groups $H_{j}$ may be
associated to linear or non-linear physical phenomena such as rotations,
scaling, colored illuminations or pitch frequency shifts. 
Let $SO(n)$ be the group of rotations.
Rigid movements
$SE(n) = \R^n \rtimes SO(n)$ is a non-commutative group,
which often includes local symmetries.
For images, $n = 2$, this group becomes a transport in
$P_1$ with $H_1 = SO(n)$ which rotates
a wavelet filter $w_{1,\hh} (u) = w_1 (r_{\hh}^{-1} u)$.
Such filters are often observed in the first layer of deep convolutional
networks \cite{fergus}. 
They map the action of 
$\bar g = (v,r_k) \in SE(n)$ on $x$ to
a parallel transport of $(u,\hh) \in P_1$ defined 
for $g \in G_1 = \R^2 \times SO(n)$ by $g.(u,\hh) = (v+r_k u,\hh+k)$.
Small diffeomorphisms in $\Diff(G_j)$
correspond to 
deformations along translations and rotations, which are sources of local
symmetries. 
A roto-translation scattering  \cite{sifre,Oyallon}
linearizes them with  wavelet filters along 
translations and rotations, with $G_j = SE(n)$ for all $j > 1$. 
This roto-translation scattering can efficiently
regress physical functionals which are often invariant to
rigid movements, and Lipschitz continuous to deformations. 
For example, quantum molecular energies $f(x)$ 
are well estimated by sparse regressions over such scattering
representations  \cite{Hirn}.

\paragraph{Audio pitch} Pitch frequency shift is a more complex example of a non-linear symmetry 
for audio signals. Two different musical notes of a same instrument have
a pitch shift. Their harmonic frequencies are multiplied by
a factor $2^\hh$, but it is not a dilation because
the note duration is not changed. 
With narrow band-pass filters  $w_{1,\hh}(u) = w_1 (2^{-\hh} u)$,
a pitch shift is approximatively mapped to
a translation along $\hh \in H_1 = \R$ of $\rho(x \star w_{1,\hh} (u))$,
with no modification along the time $u$. 
The action of $g = (v,k) \in G_1 = \R \times \R = \R^2$ over
$(u,\hh) \in P_1$ 
is thus a two-dimensional translation $g.(u,\hh) = (u+v,\hh+k)$.
A pitch shift also comes with deformations along time and log-frequencies, 
which define a much larger class of symmetries in $\Diff(G_1)$.
Two-dimensional wavelets along $(u,\hh)$ can
linearize these small time and log-frequency deformations.
These define a joint time-frequency scattering applied to 
speech and music classifications \cite{andenLostanlen}.
Such transformations were first proposed as neurophysiological
models of audition \cite{shamma}.

\paragraph{Manifolds of patterns} The group $H_{j}$ is associated to complex transformations when $j$ increases.
It needs to capture large transformations between different
patterns in a same class, for example chairs of different styles. Let us consider
training samples $\{x^i \}_i$ of a same class.
The iterated network contractions
transform them into vectors
$\{x^i_{j-1} \}_i$ which are much closer. 
Their distances define weighted graphs which sample 
underlying continuous manifolds in
the space. Such manifolds clearly appear in \cite{Aubry}, 
for high-level patterns
such as chairs or cars, together
with poses and colors. 
As opposed to manifold learning,
deep network filters result from
a global optimization which can be
computed in high dimension. The principal symmetry generators of $H_j$
is associated to common transformations over all
manifolds of examples $x^i_{j-1}$, which preserve the class
while capturing large intra-class variance.
They are approximatively mapped to a parallel transport in $x_j$ by
the filters $w_{j,\hh.b}$.
The diffeomorphisms in  $\Diff(G_j)$ are non-parallel transports
corresponding to high-dimensional displacements on the manifolds of $x_{j-1}$.
Linearizing $\Diff(G_j)$ is equivalent
to partially flatten simultaneously all these manifolds, which may
explain why manifolds are progressively
more regular as the network depth increases \cite{Aubry}, but it involves
open mathematical questions.

\paragraph{Sparse support vectors} We have up to now been concentrated on the reduction of the data variability through contractions.
We now explain why the classification margin can be preserved
thanks to the existence of multiple fibers $B_j$ in $P_j$,
by adapting filters instead of using standard wavelets.
The fibers indexed by $b \in B_j$ are separation instruments, which
increase dimensionality to avoid reducing the classification margin. 
They prevent from collapsing 
vectors in different classes, which have a distance $\|x_{j-1} -  x_{j-1}'\|$ 
close to the minimum margin $\epsilon$. These vectors are close
to classification frontiers. They are called {\it multiscale support vectors},
by analogy with support vector machines.
To avoid further contracting their distance, 
they can be separated along different fibers indexed by $b$. 
The separation is achieved by filters $w_{j,\hh.b}$, which transform
$x_{j-1}$ and $x'_{j-1}$ into 
$x_{j}(g,\hh,b)$ and
$x_j'(g,\hh,b)$ having sparse supports on different fibers $b$. 
The next contraction $\rho W_{j+1}$ reduces distances along fibers indexed
by $(g,\hh) \in G_j$, but not across
$b \in B_j$, which preserves distances.
The contraction increases with $j$ so the number of support vectors close
to frontiers also increases,
which implies that more fibers are needed to separate them.

When $j$ increases, 
the size of $x_j$ is a balance between the dimension reduction along fibers,
by subsampling $g \in G_j$, and an increasing number of fibers $B_j$ which
encode progressively more support vectors.
Coefficients in these fibers become more specialized and invariants,
as the grandmother neurons observed in deep layers of
convolutional networks \cite{grandmother}. They have a strong response to 
particular patterns and are invariant to a large class of transformations.
In this model, the choice of filters
$w_{j,\hh.b}$ are adapted to produce sparse representations
of multiscale support vectors. They provide a sparse distributed
code, defining invariant pattern memorisation. This
memorisation is numerically observed in 
deep network reconstructions \cite{fergus}, which
can restore complex patterns within each class. 
Let us emphasize that groups and fibers are mathematical ghost behind filters,
which are never computed. The learning optimization is directly performed on
filters, which carry the trade-off between contractions to reduce the data variability
and separation to preserve classification margin.

\section{Conclusion}

This paper provides a mathematical framework to analyze
contraction and separation properties of deep convolutional networks.
In this model, network filters are guiding non-linear contractions,
to reduce the data variability in directions of local symmetries. 
The classification margin can be controlled by sparse separations along network fibers.
Network fibers combine invariances along groups of symmetries
and distributed pattern representations,   
which could be sufficiently stable to explain transfer learning 
of deep networks \cite{nature}. However, this is only a framework.
We need
complexity measures, approximation theorems in spaces of high-dimensional functions,
and guaranteed convergence of filter optimization, 
to fully understand the mathematics of these convolution networks. 

Besides learning, there are striking similarities between these
multiscale mathematical tools and the treatment of 
symmetries in particle and statistical physics \cite{Glinsky}. 
One can expect a rich cross fertilization between high-dimensional
learning and physics, through the development of a common mathematical
language.

\paragraph{ Acknowledgements} I would like to thank Carmine Emanuele Cella, Ivan Dokmaninc, Sira Ferradans,  Edouard Oyallon and Ir\`ene Waldspurger for their helpful comments and suggestions.

\paragraph{Funding} This work was supported by the ERC grant InvariantClass 320959.


\end{document}